\documentclass[letterpaper, 10 pt, conference]{ieeeconf}  % Comment this line out if you need a4paper

\IEEEoverridecommandlockouts                              % This command is only needed if 
\usepackage{times} % assumes new font selection scheme installed
\usepackage{amsmath} % assumes amsmath package installed
\usepackage{amssymb}  % assumes amsmath package installed
\usepackage{dsfont}
\usepackage{color}
\usepackage{graphicx}
\usepackage{siunitx}
\usepackage{url}
\usepackage{subcaption}
\usepackage{cite}
\usepackage{comment}
\usepackage{algorithm}
\usepackage{algpseudocode}
\usepackage{algorithmicx}

\usepackage{amsthm}
\newtheorem{theorem}{Theorem}

\newtheorem{proposition}{Proposition}
\newtheorem{remark}{Remark}

\newcommand{\Figref}[1]{Figure\,\ref{#1}}
\newcommand{\figref}[1]{Fig.\,\ref{#1}}
\newcommand{\secref}[1]{Sec.\,\ref{#1}}

\newcommand{\mP}{{\mathbb{P}}}
\newcommand{\mI}{{\mathds{1}}}
\newcommand{\R}{\mathbb{R}}
\newcommand{\mZ}{\mathbb{Z}}
\newcommand{\SProb}{\Psi}
\newcommand{\SSet}{\mathcal{C}}

\DeclareMathOperator*{\argsup}{arg\,sup}

\title{\LARGE \bf
Autonomous Drifting Based on Maximal Safety Probability Learning
}

\author{
Hikaru Hoshino$^{1}$,
Jiaxing Li$^{2}$,
Arnav Menon$^{2}$,
John M. Dolan$^{3}$, 
Yorie Nakahira$^{2}$
\thanks{*This work was supported in part by Grant-in-Aid for Scientific Research (KAKENHI) from the Japan Society for Promotion of Science (\#23K13354)}% <-this % stops a space
\thanks{$^{1}$H.~Hoshino is with the Department of Electrical Materials and Engineering, University of Hyoto, 
        {\tt\small hoshino@eng.u-hyogo.ac.jp}}%
\thanks{$^{2}$J.~Li, A.~Menon, and Y.~Nakahira are with the Department of Electrical and Computer Engineering, Carnegie  Mellon  University, 
        {\tt\small \{jiaxingl, arnavmen, ynakahira\}@andrew.cmu.edu}}%
\thanks{$^{3}$J.~M.~Dolan is with the Robotics Institute, Carnegie  Mellon  University, 
        {\tt\small jdolan@andrew.cmu.edu}}%    
}

\renewcommand{\baselinestretch}{0.98}

\begin{document}

\maketitle
\thispagestyle{empty}
\pagestyle{empty}

%%%%%%%%%%%%%%%%%%%%%%%%%%%%%%%%%%%%%%%%%%%%%%%%%%%%%%%%%%%%%%%%%%%%%%%%%%%%%%%%
\begin{abstract}

This paper proposes a novel learning-based framework for autonomous driving based on the concept of maximal safety probability.  
Efficient learning requires rewards that are informative of desirable/undesirable states, but such rewards are challenging to design manually due to the difficulty of differentiating better states among many safe states. On the other hand, learning policies that maximize safety probability does not require laborious reward shaping but is numerically challenging because the algorithms must optimize policies based on binary rewards sparse in time.
Here, we show that physics-informed reinforcement learning can efficiently learn this form of maximally safe policy. Unlike existing drift control methods, our approach does not require a specific reference trajectory or complex reward shaping, and can learn safe behaviors only from sparse binary rewards. This is enabled by the use of the physics loss that plays an analogous role to reward shaping. The effectiveness of the proposed approach is demonstrated through lane keeping in a normal cornering scenario and safe drifting in a high-speed racing scenario. 

%This paper proposes a novel learning-based framework for autonomous driving based on the concept of maximal safety probability. While safe control methods based on forward invariance such as control barrier functions have been successfully applied to autonomous driving tasks, constructing a certificate function remains as a challenge for drifting/slipping vehicles. We use the safety probability as a candidate of certificate function and apply a physics-informed reinforcement learning to maximize the safety probability. Unlike existing drift control methods, our approach does not require a specific reference trajectory or complex reward shaping, and can learn safe behaviors only from sparse binary rewards of safe/unsafe events. The effectiveness of the proposed approach is demonstrated through lane keeping in a normal cornering scenario and safe drifting in a high-speed racing scenario. 

\end{abstract}

%%%%%%%%%%%%%%%%%%%%%%%%%%%%%%%%%%%%%%%%%%%%%%%%%%%%%%%%%%%%%%%%%%%%%%%%%%%%%%%%
\section{INTRODUCTION}

Driving in adverse conditions (e.g. high-speed racing or icy roads with low traction) is challenging for both human drivers and autonomous vehicles.
The vehicle operates near its handling limits in a highly nonlinear regime and has to cope with uncertainties due to unmodeled dynamics and noises in sensing and localization~\cite{Liniger2015}. 
Furthermore, it needs to have a very low response time to adapt to the rapidly changing environment~\cite{Kabzan2019amz}. 
Deterministic worst-case frameworks including robust sliding-mode control~\cite{Zhang2020} %or fuzzy logic controllers~\cite{Parra2018} 
can often be efficiently computed but require full system models and small bounded uncertainties (errors). Techniques based on set invariance, such as control barrier functions and barrier certificates~\cite{ames19,Xiao2021, Huang2021, Black2023}, are applied to vehicle systems with analytical models when these functions can be designed. Techniques based on probabilistic invariance~\cite{Gangadhar2022,gangadhar2023occlusion} can be used to generate safety certificates using samples of complex (black-box) systems in extreme environments~\cite{Gangadhar2022} and occluded environments~\cite{gangadhar2023occlusion}.  
Model Predictive Control (MPC) techniques, including stochastic MPC~\cite{Brudigam2023} and chance-constrained MPC~\cite{Carvalho2014}, exploit future predictions to account for uncertainties. 
However, as the number of possible trajectories grows exponentially to the outlook time horizon, there are often stringent tradeoffs between outlook time horizon and computation burdens. 
Thus, it is still challenging to ensure safety in adverse and uncertain conditions with lightweight algorithms suitable for onboard computation.

%Recently, safe control methods based on forward invariant sets in the state space have been applied to autonomous driving. Control Barrier Functions (CBFs)~\cite{Ames2017,ames19} can be efficiently computed and are applied to lane keeping~\cite{Ames2017} and other tasks~\cite{Xiao2021, Huang2021, Black2023}. However, CBFs are usually hand-crafted for specific systems and tasks, and it is difficult to manually design CBFs considering highly nonlinear dynamics when the vehicle is drifting or slipping in adverse conditions.

Meanwhile, within drifting control, various model-based/model-free techniques are considered. 
In \cite{Hindiyeh2014}, a bicycle model is used to describe the phenomenon of vehicle drifting, and an unstable “drift equilibrium” has been found. 
Sustained drifting has been achieved by stabilizing a drift equilibrium using various control methods, such as LQR~\cite{Bardos2020}, robust control~\cite{Xu2021}, and MPC~\cite{Bellegarda2022}. 
Hierarchical control architectures are proposed in~\cite{Yang2022,Chen2023} %presented an MPC and LQR-based  tracking controller 
for general path tracking. 
A drawback of these previous works is that they aim to stabilize a specific equilibrium or reference trajectories, and pre-computation of such a reference is required.  
As a data-driven drifting control, Probability Inference for Learning Control (PILCO) has been applied~\cite{cutler16}, but the method is only considered in a single-task setting of minimizing tracking error of a particular drift equilibrium. 
Soft actor-critic algorithm was designed to go through sharp corners in a racing circuit in simulations~\cite{Cai2020:drift_drl}. 
Twin-Delayed Deep Deterministic (TD3) Policy Gradient  is developed in \cite{Orgovan2021}, and tabular Q-learning is used in \cite{Toho2023}. 
The transference of RL agents to the real world was studied by ~\cite{Schnieders2018,Domberg2022}, using experiments of radio-controlled (RC) model cars. 
Drift parking task is considered in \cite{Leng2023}. 
However, like other RL problems~\cite{Dulac-Arnold2021}, an appropriate design of reward function is required to generalize across different drift maneuvering tasks.
The reward functions in these works tend to be complex, and the distance from a target drift equilibrium or a reference trajectory is typically used for reward shaping. 

%In this paper, we explore a novel approach to drifting control based on forward invariance in the state space. The proposed approach is motivated by a recent development of Probabilistic Safety Certificate (PSC) \cite{Wang2022}, which ensures forward invariance by using the concept of safety probability, instead of CBFs, and has been applied to a driving task of path following without slipping \cite{Gangadhar2022}. While the original PSC uses the safety probability of the closed-loop system under a pre-designed nominal controller to safeguard the system, this paper aims to find a control policy that maximizes the safety probability such that the vehicle can safely drift or slip in adverse conditions. 
In this paper, we propose a novel approach to control vehicle drifting based on Physics-Informed Reinforcement Learning (PIRL). Specifically, PIRL is used to learn a control policy that maximizes the safety probability such that the vehicle can safely drift or slip in adverse conditions. Major technical challenges stem from the fact that the objective function associated with this problem takes multiplicative or maximum costs over time and that the reward is binary and sparse in time. To overcome the difficulty of learning multiplicative or maximum costs, we first present a transformation that converts this problem into RL with additive costs. To efficiently learn from sparse reward, we built upon our previous work ~\cite{Hoshino2024:ACC}, which imposes physics constraints on the loss function. This term in the loss function plays an analogous role in rewarding shaping and allows safe policies to be learned only from binary rewards that are sparse in time. The proposed framework only requires safe events (safe regions) to be specified without the need for reference trajectories or laborious reward shaping.  
%While numerical construction methods for CBFs such as sum-of-squares (SoS) techniques scale poorly to higher-dimensional systems \cite{Dawson2023}, RL techniques is expected to better scale to the dimension~\cite{Hsu2021}, and this paper presents application to a 16-dimensional problem. 
%
%
%
%Compared with the existing methods mentioned above, the proposed approach has the following advantages:
% \begin{itemize} 
%  \item A safe control policy is learned to maximize the safety probability. Unlike existing drift control methods including standard RL algorithms, PIRL does not require any reference trajectory or reward shaping and can learn only from sparse binary rewards of safe/unsafe events. 
 %Furthermore, when the learned safety probability and policy is used to safeguard the system,  the resultant system is expected to be less conservative. 
%  \item Although CBFs are usually hand-crafted and difficult to be designed for drifting vehicle, PIRL numerically characterizes the forward invariance. 
%   While numerical construction methods for CBFs such as sum-of-squares (SoS) techniques scale poorly to higher-dimensional systems \cite{Dawson2023}, RL techniques is expected to better scale to the dimension~\cite{Hsu2021}, and this paper presents application to a 16-dimensional problem. 
% \end{itemize}
The effectiveness of the proposed approach is demonstrated through lane keeping in a normal cornering scenario and safe drifting in a high-speed racing scenario.  

\begin{figure*}[t!]
    \centering
    \includegraphics[width=0.9\linewidth]{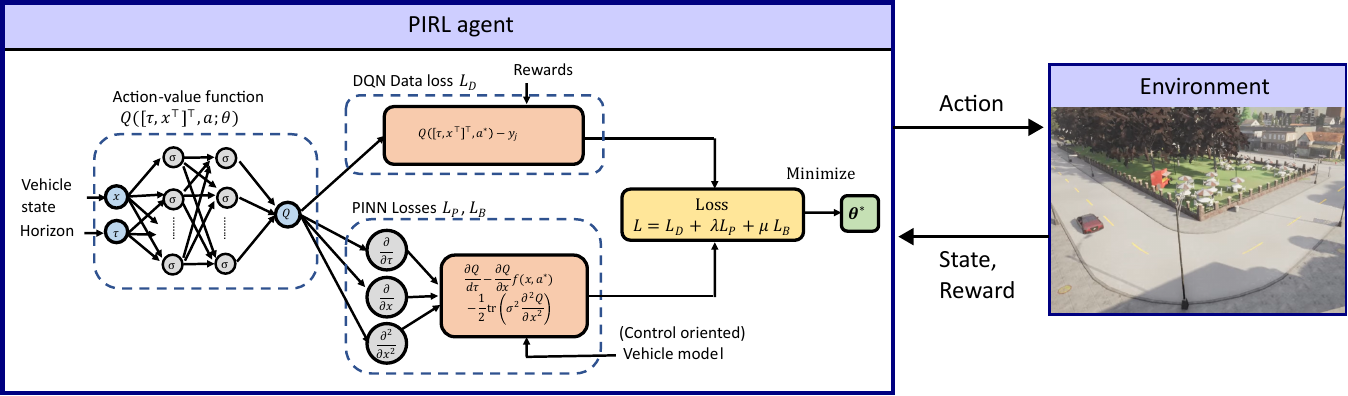}
    \caption{Framework of training by Physics-informed reinforcement learning (PIRL)}
    \label{fig:PIRL_framework}
\end{figure*}

\subsection{Notation}

Let $\mathbb{R}$ and $\mathbb{R}_+$be the set of real numbers and the set of nonnegative real numbers, respectively. 
Let $\mZ$ and $\mZ_+$ be the set of integers and the set of non-negative integers. 
For a set $A$, $A^\mathrm{c}$ stands for the complement of $A$, and $\partial A$ for the boundary of $A$. 
Let $\lfloor x \rfloor \in \mZ$ be the greatest integer less than or equal to $x\in\R$.  
Let $\mathds{1}[\mathcal{E}]$ be an indicator function, which takes 1 when the condition $\mathcal{E}$ holds and otherwise 0.
Let $\mP[ \mathcal{E} | X_0 = x ]$ represent the probability that the condition $\mathcal{E}$ holds involving a stochastic process $X = \{ X_t \}_{t\in\R_+}$ conditioned on $X_0 = x$.  
Given random variables $X$ and $Y$, let $\mathbb{E}[X]$ be the expectation of $X$ and $\mathbb{E}[X|Y=y]$ be the conditional expectation of $X$ given $Y=y$. We use upper-case letters (\emph{e.g.}, $Y$) to denote random variables and lower-case letters (\emph{e.g.}, $y$) to denote their specific realizations. 
For a scalar function $\phi$, % with an argument $x \in\R^n$, 
$\partial_x \phi$ stands for the gradient of $\phi$ with respect to $x$, and $\partial_x^2 \phi$ for the Hessian matrix of $\phi$. 
Let $\mathrm{tr}(M)$ be the trace of the matrix $M$. 
For $x \in \R^n$ and $A \subset \R^n$, $\text{dist}(x, A) := \inf_{y \in A} \| x - y \|$.

\section{Maximal Safety Probability Learning}

In this section, the problem formulation of estimating maximal safety probability is given in \secref{sec:problem_formulation}, and the PIRL algorithm \cite{Hoshino2024:ACC} is briefly reviewed in  \secref{sec:PIRL}.

%%%%
\subsection{Problem Formulation} \label{sec:problem_formulation}

We assume that the vehicle dynamics can be represented as a control system with stochastic noise of $w$-dimensional Brownian motion $\{ W_t \}_{t \in \R_+}$ starting from $W_0 = 0$. 
The system state $X_t \in \mathbb{X} \subset \mathbb{R}^n$ is assumed to be observable and evolves according to the following stochastic differential equation (SDE): 
\begin{align}
 \mathrm{d}X_t =  f(X_t, U_t) \mathrm{d}t + \sigma(X_t, U_t) \mathrm{d}W_t,  \label{eq:sde}   
\end{align}
%\todo{State that $X_t$ is observable. Also, the objective should be to estimate the maximal safety probability and to learn the corresponding safe control policy according to the introduction.}
where $U_t \in \mathbb{U} \subset \mathbb{R}^m$ is the control input. 
Throughout this paper, we assume sufficient regularity in the coefficients of the system \eqref{eq:sde}.
That is, the functions $f$ and $\sigma$ are chosen in a way such that the SDE \eqref{eq:sde} admits a unique strong solution (see, e.g., Section IV.2 of \cite{Fleming06}). 
The size of $\sigma(X_t, U_t)$ is determined from the uncertainties in the disturbance, unmodeled dynamics, and prediction errors of the environmental variables. 
For numerical approximations of the solutions of the SDE and optimal control problems, we consider a discretization with respect to time with a constant step size $\Delta t$ under piecewise constant control processes. 
For $0=t_0 < t_1 < \dots < t_k < \dots$, where $t_k := k\Delta t$, $k \in \mZ_+$, by defining the discrete-time state $X_k := X_{t_k}$ with an abuse of notation, the discretized system can be given as
\begin{align}
  X_{k+1} = F^\pi(X_k, \Delta W_k), \label{eq:descrite_system}
\end{align}
where $\Delta W_k := \{ W_t \}_{t \in [t_k, t_{k+1})}$, and $F^{\pi}$ stands for the state transition map derived from \eqref{eq:sde} under a Markov control policy $\pi : [0, \infty) \times \mathbb{X} \to \mathbb{U}$. 
%From an optimal control perspective, using a Markov policy is not restrictive when the value function has a sufficient smoothness under several technical conditions (see Assumption\,1 in \cite{Hoshino2024:ACC}). 
%Note that using a piecewise constant control process with a Markov policy $u$ implies that the control process is given as $U_t = u(\delta(t), X_{\delta(t)})$, for $t\in \R_+$,  where $\delta(t) := \lfloor t/\Delta t \rfloor \Delta t$, and the descritezed system \eqref{eq:descrite_system} has the Markov property at the discrete times \cite{Mao2013}. 

%
Safety of the system can be defined by using a safe set $\SSet \subset \mathbb{X}$.  
%For the discretized system \eqref{eq:descrite_system} and for a given sequence of control actions $\{ a_k \}_{k=0}^\infty$, 
For the discretized system \eqref{eq:descrite_system} and for a given control policy $\pi$, 
the safety probability $\SProb^{\pi}$ of %the system \eqref{eq:descrite_system} 
the initial state $X_0 = x$ for the outlook horizon $\tau \in \R$ %, with some integer $N \in \mZ_+$, 
can be characterized as %is defined as 
the probability that the state $X_k$ stays within the safe set $\SSet$ for %during the interval 
$k \in \mathcal{N}_\tau := \{0, \dots, N(\tau)\}$, where $N(\tau) := \lfloor \tau / \Delta t \rfloor$, \emph{i.e.}, 
% \begin{align}
%     \SProb(T, x; \{ a_k \}_{k=0}^\infty) := & \mP[ Z_k \in \SSet,\, \forall k \in \mathcal{N} ~|~ \notag \\
%      & \quad Z_0 = x,\, A_k = a_k,  \forall k \in \mZ_+ ]. %\mathcal{N}  ]. 
% \end{align}
\begin{align} \label{eq:safety_probability}
\SProb^\pi(\tau, x)
:= \mP[ X_k \in \SSet,\, \forall k \in \mathcal{N}_\tau ~|~ X_0 = x, \pi ].  
\end{align}
Then, the objective of the learning is to estimate the maximal safety probability defined as 
   \begin{align}
     \Psi^\ast(\tau,x) := \sup_{ \pi \in \mathcal{P} } \Psi^\pi(\tau,x),
   \end{align}
where $\mathcal{P}$ is the class of bounded and Borel measurable Markov control policies, and to learn the corresponding safe control policy $\pi^\ast := \argsup_{\pi \in \mathcal{P}} \Psi^\pi$.    
%\todo{control objective isn't stated. as evaluation is done to demonstrate 'safety'. 
%if the objective is to 'quantify risk', evaluation should be done to compare with accuracy of the risk estimate.}

%%%%%%%%%%%%%%%%%%%%%%%%%%%%%%%
\subsection{Physics-informed RL (PIRL)} \label{sec:PIRL}

% The problem of estimating the maximal safety probability is naively understood as a stochastic optimal control problem with \emph{multiplicative} costs whose optimality conditions
% were characterized in~\cite{Abate2008,Summers2010}.
% However, solving such optimization problems is not trivial, particularly for high-dimensional systems.
% In this paper, we apply Physics-informed RL (PIRL) for risk quantification proposed in~\cite{Hoshino2024:ACC}, where the multiplicative costs are transferred to additive costs by a state augmentation, and the resultant RL problem is solved in combination with Physics-informed Neural Network (PINN) techniques~\cite{Raissi2019}, which provides theoretical guarantees on error bounds of estimated safety probabilities~\cite{Wang2023}. 

% [Theorem 6]
% for estimation accuracy. 
%demonstrated potential in generalization due to the use of physics constraints~\cite{Cai2021,Cuomo2022}, and 
%PINN-based approach has been used to quantify safety probabilities of a given controller with provable generalization~\cite[Theorem 6]{Wang2023}.

The training framework is illustrated in \figref{fig:PIRL_framework}. 
PIRL integrates RL and Physics-Informed Neural Networks (PINNs)~\cite{Raissi2019}
for efficient estimation of maximal safety probability. 
This paper uses a PIRL algorithm based on Deep Q-Network (DQN) \cite{mnih15}.
The overall structure follows from the standard DQN algorithm, and the  optimal action-value function will be estimated by using a function approximator. % with the parameter $\theta$. 
For this function approximator,  we use a PINN, which is a neural network trained by penalizing the discrepancy from a partial differential equation (PDE) condition that the safety probability should satisfy. 
To define an appropriate RL problem, we consider the augmented state space $\mathcal{S} := \R \times \mathbb{X} \subset \R^{n+1}$ and the augmented state $S_k \in \mathcal{S}$, where we denote the first element of $S_k$ by $H_k$ and the other elements by $X_k$, \emph{i.e.}, 
\begin{align}
 S_k =[ H_k, X_k^\top]^\top,  \label{eq:augmented_state}
\end{align}
where $H_k$ represents the remaining time before the outlook horizon $\tau$ is reached. 
Then, consider the stochastic dynamics starting from the initial state $ s := [\tau, x^\top]^\top \in \mathcal{S}$ given as follows%
\footnote{With this augmentation, the variable $X_k$ in $S_k$ is no longer equivalent to that in \eqref{eq:descrite_system}, because $X_k$ transitions to itself when $S_k \in \mathcal{S}_\mathrm{abs}$. However, we use the same notation for simplicity. See \cite{Hoshino2024:ACC} for precise derivation.}: 
for $\forall k \in \mZ_+$, 
\begin{align}
  S_{k+1} = 
  \begin{cases}
      \tilde{F}^\pi(S_k, \Delta W_k),  & S_k \notin \mathcal{S}_\mathrm{abs},   \\
      S_k,  & S_k \in \mathcal{S}_\mathrm{abs}, 
  \end{cases}
  %\quad \forall k \in \mZ_+
 \label{eq:augmented_dynamics}
 \end{align}
with the function $\tilde{F}^\pi$ given by 
\begin{align}
  \tilde{F}^\pi(S_k, \Delta W_k) := 
 \begin{bmatrix} 
    H_k - \Delta t \\
    F^\pi( X_k, \Delta W_k) 
 \end{bmatrix},   
\end{align}
and the set of absorbing states $\mathcal{S}_\mathrm{abs}$ given by 
\begin{align}
 \mathcal{S}_\mathrm{abs} := \{ [ \tau, x^\top]^\top \in \mathcal{S} ~|~  \tau < 0 \, \lor \, x \notin \SSet   \}.
 \label{eq:terminal_states}
\end{align}
Then, the following proposition states that the value function of our RL problem is equivalent to the safety probability. 
%
% 
% Proposition 1
\begin{proposition} \label{prop:RL}
Consider the system \eqref{eq:augmented_dynamics} starting from an initial state $s = [\tau, x^\top]^\top \in \mathcal{S}$ %for $S_k \notin \mathcal{S}_\mathrm{abs}$ with the absorbing states given in \eqref{eq:terminal_states} 
and the reward function $r: \mathcal{S} \to \R$ given by
\begin{align} \label{eq:reward}
  r(S_k) := 
      \mI[ H_k \in \mathcal{G} ]
\end{align}
with $\mathcal{G} := [0, \Delta t)$. 
Then, for a given control policy $\pi$, the value function $v^{\pi}$ defined by  
\begin{align}
  &v^{\pi}(s) := \mathbb{E} \left[ \sum_{k=0}^{N_\mathrm{f}} r( S_k ) ~\middle | ~ S_0= s, \pi  \right],  \label{eq:safe_ocp_cost}
\end{align}
where $N_\mathrm{f} := \inf \{ j \in \mZ_+ \,|\, S_j \in \mathcal{S}_\mathrm{abs} \}$,
takes a value in $[0,1]$ and is equivalent to the safety probability $\SProb^\pi(\tau, x)$, i.e., 
\begin{align}
 v^{\pi}(s) = \SProb^{\pi}(\tau, x).     
\end{align}
\end{proposition}
\begin{proof}
  The proof follows from Proposition\,1 of \cite{Hoshino2024:ACC} with a minor rewrite of the reward and value functions.   
\end{proof}

Thus, the problem of safety probability estimation can be solved by an episodic RL problem with the action-value function $q^\pi(s, u)$, defined as the value of taking an action $u \in \mathbb{U}$ in state $s$ and thereafter following the policy $\pi$:
\begin{align}
 q^\pi(s, u) := \mathbb{E}\left[ \sum_{k=0}^{N_\mathrm{f}} r(S_k) ~\middle|~ S_0 = s,\, U_0 = u,\, \pi  \right].
\end{align}
% The objective of RL is to find the optimal action-value function defined as 
% \begin{align}
%   q^\ast(s, a) := \sup_{ u \in \mathcal{U} } q^u(s, a).
% \end{align}
%
%
Furthermore, it is shown in \cite{Hoshino2024:ACC} that the maximal safety probability can be characterized by the Hamilton-Jacobi-Bellman (HJB) equation of a stochastic optimal control problem.
While the HJB equation can have discontinuous viscosity solutions, the following theorem shows that a slightly conservative but arbitrarily precise approximation can be constructed by considering a set $\SSet_\epsilon$ smaller than $\SSet$: 
%
%
% Although RL has the potential to offer scalable risk quantification techniques, one may not know how accurate the converged solutions are, and there is no guarantee about generalization to states or time horizons whose samples are missing. 
% This is problematic if the quantified risk is to be used in safety-critical systems, because the safety of subsequent decision-making techniques depends on accurate risk quantification. 
% To cope with this problem, PIRL leverages the Physics-informed neural networks (PINNs), which are neural networks that are trained to solve nonlinear partial differential equations~\cite{Raissi2019}. 
% PIRL uses the fact that the maximal safety probability can be characterized by the Hamilton-Jacobi-Bellman (HJB) equation of a stochastic optimal control problem~\cite{MohajerinEsfahani2016}. 
% Although the HJB equation characterizing the maximal safety probability  does not admit a classical solution, which implies that we cannot naively use numerical techniques that require continuity or smoothness of the solution, it is known that one can construct a slightly conservative but arbitrarily precise way of characterizing the original solution by considering a set $\SSet_\epsilon$ smaller than $\SSet$: 
%
%
%%%%%%%%%%%%%%%%%%%%%%%%%%%%%%
% Theorem 1
%\todo{in "if the Assumption\,1 in \cite{Hoshino2024:ACC} holds" can you state the exact assumptions in here}
\begin{theorem}[\cite{Hoshino2024:ACC}] \label{thm:hjb}
Consider the system \eqref{eq:augmented_dynamics} derived from the SDE \eqref{eq:sde} and the action-value function $q^\pi_\epsilon(s,u)$ given by 
\begin{align}
 q^\pi_\epsilon(s, u) := \mathbb{E}\left[ \sum_{k=0}^{N_\mathrm{f}} r_\epsilon(S_k) ~\middle|~ S_0 = s,\, U_0 = u,\, \pi  \right], 
\end{align}
where  the function $r_\epsilon$ is given by 
\begin{align}
  r_\epsilon(S_k) := 
      \mI[ H_k \in \mathcal{G} ]\, l_\epsilon( X_k ),
\end{align}
with $\SSet_\epsilon := \{ x \in \SSet \,|\, \mathrm{dist}(x, \SSet^\mathrm{c}) \ge \epsilon \}$ and 
\begin{align}
 l_\epsilon (x) :=  \max \left\{  1- \dfrac{\mathrm{dist}(x, \SSet_\epsilon)}{\epsilon}, \, 0\right\}.  
\end{align}
Then, if the Assumption\,1 in \cite{Hoshino2024:ACC} holds, the optimal action-value function $q^\epsilon(s,u): = \sup_{ u \in \mathbb{U} } q^\pi_\epsilon(s,u)$ becomes a continuous viscosity solution of the following PDE in the limit of $\Delta t \to 0$: 
\begin{align} 
      & %\dfrac{ \partial q^\ast_\epsilon(s, a^\ast)}{ \partial s }
      \partial_s  q^\ast_\epsilon(s, u^\ast) 
      %\bigg|_{s=s, a=a^\ast} 
      \tilde{f}(s, u^\ast) 
      \notag \\ & \hspace{5mm} 
      +\dfrac{1}{2} \mathrm{tr} \left[ \tilde{\sigma}(s,u^\ast)\tilde{\sigma}(s,u^\ast)^\top  \partial_s^2 q^\ast_\epsilon(s, u^\ast)  %\bigg|_{s=s, a=a^\ast} 
      \right] = 0,  \label{eq:pde}
\end{align} 
where the function $\tilde{f}$ and $\tilde{\sigma}$ are given by  
\begin{align}
    \tilde{f}(s, u) := 
    \begin{bmatrix}
       -1 \\   f(x, u)
    \end{bmatrix},
    \quad
    \tilde{\sigma}(s, u) := 
    \begin{bmatrix}
        0 \\ \sigma(x, u)        
    \end{bmatrix}, 
\end{align}
and $u^\ast := \argsup_{u \in \mathbb{U}} q^\ast(s, u)$. 
The boundary conditions are given by 
\begin{align}
  & q^\ast_\epsilon([0, x^\top]^\top, u^\ast) = l_\epsilon( x ), \quad  \forall x \in \mathbb{X}, \label{eq:boundary_tau_zero} \\
    & q^\ast_\epsilon([\tau, x^\top]^\top, u^\ast) = 0, \quad   \forall \tau \in \R, ~ \forall x \in \partial\SSet.    \label{eq:boundary_unsafe}
\end{align}
\end{theorem}

\begin{remark}
  When we assume further regularity conditions on the function $l_\epsilon(x)$ (i.e., differentiability of $l_\epsilon(x)$), the PDE \eqref{eq:pde} can be understood in the classical sense (see e.g., \cite[Theorem IV.4.1]{Fleming06}). This means that the PDE condition can be imposed by the technique of PINN using automatic differentiation of neural networks. 
\end{remark}

% PIRL integrates RL and PINN techniques to accurately estimate the safety probability. 
% This paper uses a PIRL algorithm based on Deep Q-Network (DQN) \cite{mnih15}.
% The algorithm statement is provided in Appendix\,A. 
% The overall structure follows from the standard DQN algorithm, and the  optimal action-value function $q^\ast(s,a)$ will be estimated by using a function approximator $Q(s, a; \theta)$ with the parameter $\theta$. 
% For this function approximator,  we use a PINN, which is a neural network trained by penalizing the discrepancy from the PDE condition \eqref{eq:pde} that the safety probability should satisfy. 
Based on the above, a standard DQN algorithm is integrated with a PINN by using a  modified loss function given by 
\begin{align}
    L = & L_\mathrm{D} + \lambda L_\mathrm{P}
    + \mu L_\mathrm{B}, \label{eq:loss_DQN+PINN} 
\end{align}
where $L_\mathrm{D}$ is the data loss of the original DQN, $L_\mathrm{P}$ the loss term imposing the PDE \eqref{eq:pde}, and $L_\mathrm{B}$ the loss term for the boundary conditions \eqref{eq:boundary_tau_zero} and \eqref{eq:boundary_unsafe}. 
The parameters $\lambda$ and $\mu$ are the weighting coefficients for the regularization loss terms $L_\mathrm{P}$ and $L_\mathrm{B}$, respectively. 
Detailed descriptions of these loss terms are given in \cite{Hoshino2024:ACC}.

%%%%%%%%%%%%%%%%%%%%%%%%%%%%%%%%%%%%%%%%
\section{Numerical Experiments}

This section discusses numerical results of autonomous driving based on maximal safety probability learning. 
After presenting training methodology in \secref{sec:setup}, we provide results of normal cornering in \secref{sec:lane_keeping} and high-speed drifting in \secref{sec:drifting}. 
Our implementation is available at \url{https://github.com/hoshino06/pirl_itsc2024}.

%%%%%%%
\subsection{Training Methodology} \label{sec:setup}

The PIRL agent learns maximal safety probabilities by interacting with CARLA~\cite{Dosovitskiy2017}, an open-source simulator providing high-fidelity vehicle physics simulations that are not explicitly described by control-oriented models such as a bicycle model. 
The input to the neural network as a function approximator $Q$ for our action-value function is the state $s \in \mathcal{S}$ defined in \eqref{eq:augmented_state}, and the output is a vector of q values for all discrete actions as explained below. 
The state $s$ consists of the vehicle state $x$ and the outlook horizon $\tau$:
\begin{align}
 s = [\tau, x^\top]^\top,
\end{align}
and $x$ can be further decomposed into $x_\mathrm{vehicle}$ for the vehicle dynamics and $x_\mathrm{road}$ for the vehicle position relative to the road:
\begin{align}
 x = 
 [ \,\underbrace{v_x, \, \beta, \, r}_{x_\mathrm{vehicle}}, \,\underbrace{ e, \,\psi,  \mathcal{T} }_{x_\mathrm{road}} \, ]^\top \in \mathbb{R}^{15},  
\end{align}
where $x_\mathrm{vehicle}$ consists of $v_x$ standing for the vehicle longitudinal velocity, $\beta$ for the sideslip angle, defined as $\beta := \arctan(v_y/v_x)$ with the lateral velocity $v_y$, and $r$ for the yaw rate. 
The variable $x_\mathrm{road}$ consists of $e$ standing for the lateral error from the center line of the road, and $\psi$ for the heading error, and $\mathcal{T}$ that contains five $(x, y)$ positions of reference points ahead of the vehicle placed on the center line of the lane.
%Therefore, the dimension of the state space $\mathcal{S}$ is 16. 
%
%
The control action $a$ is given by 
\begin{align}
%    u = [\Delta \delta, \, \tau_\mathrm{e}  ]^\top
    a = [ \delta, \,  d  ]^\top, 
\end{align}
where $\delta$ is the steering angle, and $d$ is the throttle.
They are normalized to $[-1, 1]$ and $[0, 1]$, respectively. 
To implement our DQN-based algorithm, which admits a discrete action space, the control inputs are restricted to $d \in \{ 0.6,\, 0.7,\, 0.8,\, 0.9,\, 1.0  \}$ and $\delta \in \{ -0.8, \, -0.4, \, 0, \, 0.4, \, 0.8\}$. 
The reason for limiting $d \ge 0.6$ is to prevent slow driving with trivially safe behaviors (the vehicle is safe if it never moves). 
The steering is limited to $|\delta| \le 0.8$, since high-speed vehicles are prone to rollover if large steering angles are applied \cite{Cai2020:drift_drl}.

For better sample efficiency and generalization to unseen regions, PIRL can exploit the structure of a control-oriented model encoded in the form of a PDE in \eqref{eq:pde} and the regularization loss term $L_\mathrm{P}$ in \eqref{eq:loss_DQN+PINN}. 
In this paper, we used a dynamic bicycle model with a 
Pacejka tire model. % (see Appendix\,B for details). 
The parameters of the model that correspond to the vehicle dynamics in CARLA are assumed to be known in this paper, but PINN can also be used for the discovery of coefficients of PDEs from data~\cite{Raissi2019}, or the entire tire model can be obtained from online learning~\cite{Kalaria2023:APAC}.   
Also, the term $\sigma(X_t, U_t)$ and the convection term for the reference points $\mathcal{T}$ were neglected in the simulations, which need to be constructed or learned from data for more precise estimation. 

% \begin{figure*}[!t]
%     \centering
%     \includegraphics[width=0.95\linewidth]{figs/maps.png}
%     \caption{Seven maps for training and testing \cite{Cai2020:drift_drl}. The tracks (a-f) are for training and the track (g) for evaluation.}
%     \label{fig:maps}
% \end{figure*}

%%%%%%%%%%%%%%%%%%
\subsection{Normal Cornering} \label{sec:lane_keeping}

We first show preliminary results with a normal cornering task. 
The objective is to keep the vehicle within the lane while driving on a road. 
The safe set $\SSet \subset \mathbb{X}$ is given by 
\begin{align}
 \SSet = \{ x \in \mathbb{X} ~|~ |e| \le E_\mathrm{max} \}, 
 \label{eq:safe_set}
\end{align}
where $E_\mathrm{max}$ is the maximum distance from the center line of the lane, and it is assumed to be constant. 
We used a corner in a built-in map of the CARLA simulator (southwest part of Town2 and shown in \figref{fig:town2_south_west}), and the bound of the lateral error is set to $E_\mathrm{max} = \SI{1}{m}$. 
During the training, the initial vehicle speed was randomly selected within $[\SI{5}{m/s}, \SI{15}{m/s}]$, and the vehicle spawn point was randomized for each episode. 
Also, the outlook horizon $\tau$ was uniformly randomized between $0$ and $\SI{5.0}{s}$. 
For the function approximator $Q$, we used a neural network with 3 hidden layers with 32 units per layer and the hyperbolic tangent (\texttt{tanh}) as the activation function. 
At the output layer, the sigmoid activation function was used.

%%%%%%%%%
% Figure: normal cornaring in Town2
\begin{figure}[!t]
    \centering
    \includegraphics[width=0.7\linewidth]{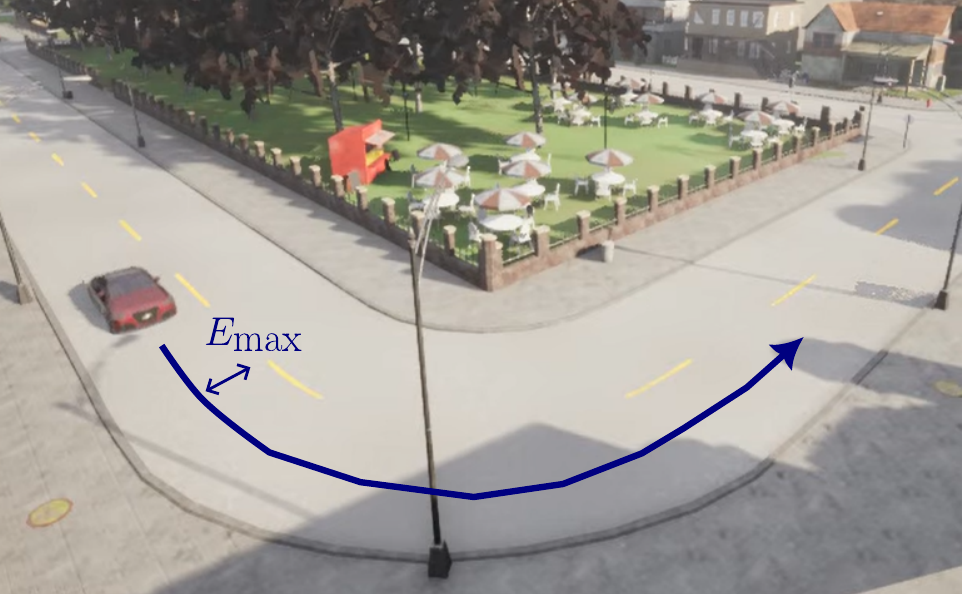}
    \caption{Lane keeping with normal cornering}
    \label{fig:town2_south_west}
\end{figure}

\Figref{fig:town2_training_progress} shows the training progress with the learning rate of $5\times10^{-4}$ and the regularization weight of $1 \times 10^{-4}$. 
The \emph{orange} and \emph{blue} lines show the episode reward and the q-value averaged over a moving window of 500 episodes.   
The solid curves represent the mean of 10 repeated experiments, and the shaded region shows their standard deviation. 
While there is a large variance during the transient phase of learning, it converges to similar values after about 20,000 episodes.
%
%
%%%%%%%%%%%%%%%%%%%%%
% Fig: Town2 training progress
\begin{figure}[!t]
    \centering
    \includegraphics[width=0.8\linewidth]{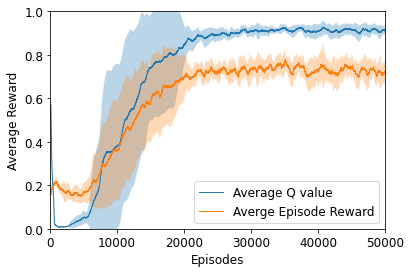}
    \caption{Training progress for normal cornering task}
    \label{fig:town2_training_progress}
\end{figure}
%
%
%
%\Figref{fig:town2_safety_probability} shows the learned safety probability of a learned agent. 
As defined in \eqref{eq:safety_probability}, the safety probability depends on the initial state and the horizon $\tau$. 
%lateral error $e$, the relative heading angle $\psi$, and the waypoints $\mathcal{T}$ ahead of the vehicle, as well as the vehicle state $x_\mathrm{vehicle}$ and the horizon $\tau$. 
\Figref{fig:SafeProb_e_psi} shows the dependency of the safety probability on the lateral error $e$ and the relative heading angle $\psi$, when the vehicle is placed on the straight part of the road. 
The vehicle state and the horizon are fixed to $x_\mathrm{vehicle} = [\SI{10}{m/s}, \,0,\, 0]^\top$ and $\tau = \SI{5.0}{s}$. 
The learned safety probability is as high as approximately 1 near $(e, \psi)=(0,0)$, and decreases near the boundaries of the road ($|e|=\SI{1}{m}$). 
Also, it decreases when the heading angle is large.  
Similarly, \figref{fig:town2_safety_probability}b shows that the safety probability decreases as the velocity $v_x$ increases and is higher in the inner side of the curve ($e < 0$). 
Overall, the above result appropriately describes how the safety probability depends on the vehicle state and position in the lane. 
However, the values of the safety probability differed across agents, and further investigation is needed to ensure the accuracy. 
Nevertheless, it has been confirmed that the learned agents safely go through the corner when they are placed near the center of the road.

%%%%%%%%%%%%%%%%%%%%%%
% Fig: Town2 safety probability
% \begin{figure}[!t]
%     \centering
%     \vspace{-1mm}
%      \subcaptionbox{Straight part of the road}{\includegraphics[width=0.9\linewidth]{figs/safe_prob_e_psi.png}} \\[2mm]
%     \subcaptionbox{At the corner of the road}{\includegraphics[width=0.9\linewidth]{figs/safe_prob_e_v.png}}
%     \caption{Learned safety probability}
%     \label{fig:town2_safety_probability}
% \end{figure}
\begin{figure}[t]
  \begin{minipage}[b]{0.49\linewidth}
    \centering
    \includegraphics[width=\linewidth]{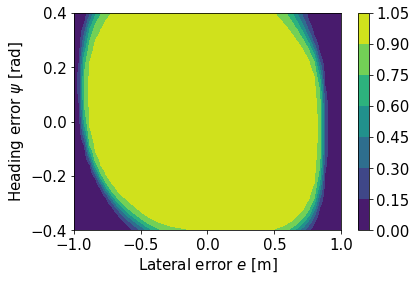}
    \subcaption{{Effect of $e$ and $\psi$ }}\label{fig:SafeProb_e_psi}
    \hspace*{5mm}(at straight road)
  \end{minipage}
  \begin{minipage}[b]{0.49\linewidth}
    \centering
    \includegraphics[width=\linewidth]{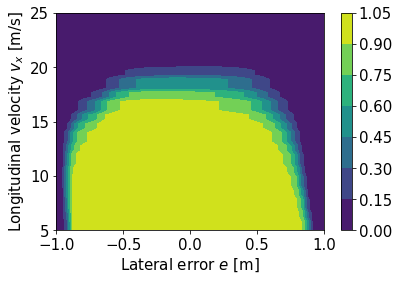}
    \subcaption{Effect of $e$ and $v_x$}\label{fig:SafeProb_e_psi}
    \hspace*{1mm}(at the corner)
  \end{minipage}
  \caption{Learned safety probability}\label{fig:town2_safety_probability}
\end{figure}

%The unsafe event is defined by lane violation constraints for both left and right lane boundaries. 
%Six maps with various levels of difficulty (\figref{fig:maps}(a-f)) are used for training, and the performance of the trained agent is evaluated with an unseen test track (\figref{fig:maps}(g)). 
%These maps were designed in \cite{Cai2020:drift_drl} based on the tracks of a racing game.  

%%%%%%%%%%%%%%%%%%
\subsection{Safe Drifting} \label{sec:drifting}

Here we present results for safe drifting based on maximal safety probability learning. 
A racing circuit for CARLA developed in \cite{Cai2020:drift_drl} is used, and the training is performed at a specific corner as shown in \figref{fig:mapC_task}. 
The road width is about \SI{20}{m} and $E_\mathrm{max}$ can be $\SI{10.0}{m}$, but the safe set was conservatively set as $E_\mathrm{max} = \SI{8.0}{m}$ during the training.
%The safe set is given by \eqref{eq:safe_set} with $E_\mathrm{max} = \SI{8.0}{m}$. 
%\figref{fig:mapC_task} illustrates an example of the vehicle  drifting through the corner. 
%During the training, t
The initial vehicle speed was set to $\SI{30}{m/s}$, and the slip angle $\beta$ and the yaw rate $r$ were randomized in the range of 
\begin{align}
\color{blue}
    &\beta \in [ -\SI{25}{deg}, -\SI{20}{deg} ], 
    \\ & 
    r \in [ \SI{50}{deg/s}, \SI{70}{deg/s} ].  
\end{align}%
The neural network has the same shape with that used above. 
%
%
% Here we explore safety probabilities of recovery from unusual accidental conditions (\figref{fig:mapC_task}). 
% To quantify these, the vehicle was put in initial conditions where the slip angle $\beta$ and/or the yaw rate $r$ are higher than those values seen in drifting by expert drivers: 
% {\color{blue}
% \begin{align}
% \color{blue}
%     &\beta \in [ -\SI{30}{deg}, \SI{30}{deg} ], 
%     \\ & 
%     r \in [ -\SI{360}{deg/s}, \SI{360}{deg/s} ].  
% \end{align}}%
% The longitudinal speed was $v_x = {\color{blue}\SI{20}{m/s}}$. 
% This setting corresponds to a situation where the vehicle is accidentally get into an unusual state, and recovery from such an initial condition is a challenging task because there is no clear guidelines of doing reward shaping nor expert driving trajectories to follow. 
%
%
%
For this task, the training progress is strongly influenced by the learning rate as shown in \figref{fig:mapC_training_progress}. 
As before, the solid curves represent the mean of the averaged rewards of 8 repeated experiments, and the shaded regions show their standard deviations. 
%learning rate and the regularization weight $\lambda$ are  $2\times10^{-4}$ and $1 \times 10^{-4}$, respectively, and the training progress is shown in \figref{fig:mapC_training_progress}. 
%\figref{fig:mapC_training_progress} shows training progress and the average return reaches as high as 0.9. 
At the initial phase of training, the averaged reward (empirical safety probability) gradually decreases, but after a while it starts to increase.
The timing of this increase is faster for larger learning rates, but the final episode reward increases as the learning rate is reduced from $\SI{2e-4}{}$ to $\SI{5e-5}{}$. 
However, at the learning rate of $\SI{2e-5}{}$, the agents failed to learn safe policy and the reward dropped to zero except for 2 cases out of 8 experiments.
Also, while the learning rate of $\SI{5e-5}{}$ performed best among the four settings above, there were few cases where the reward dropped to zero, which were excluded from the plot in \figref{fig:mapC_training_progress}. 
After these trials, we selected an agent trained with the learning rate $\SI{5e-5}{}$ at a check point of 90,000 episodes that behaved best when tested with closed-loop simulations. 

%%%%%%%%%%%%%%%%%%%%%%%%%%%%%%
% Fig: Drifting in MapC
\begin{figure}[!t]
    \centering
    \includegraphics[width=0.8\linewidth]{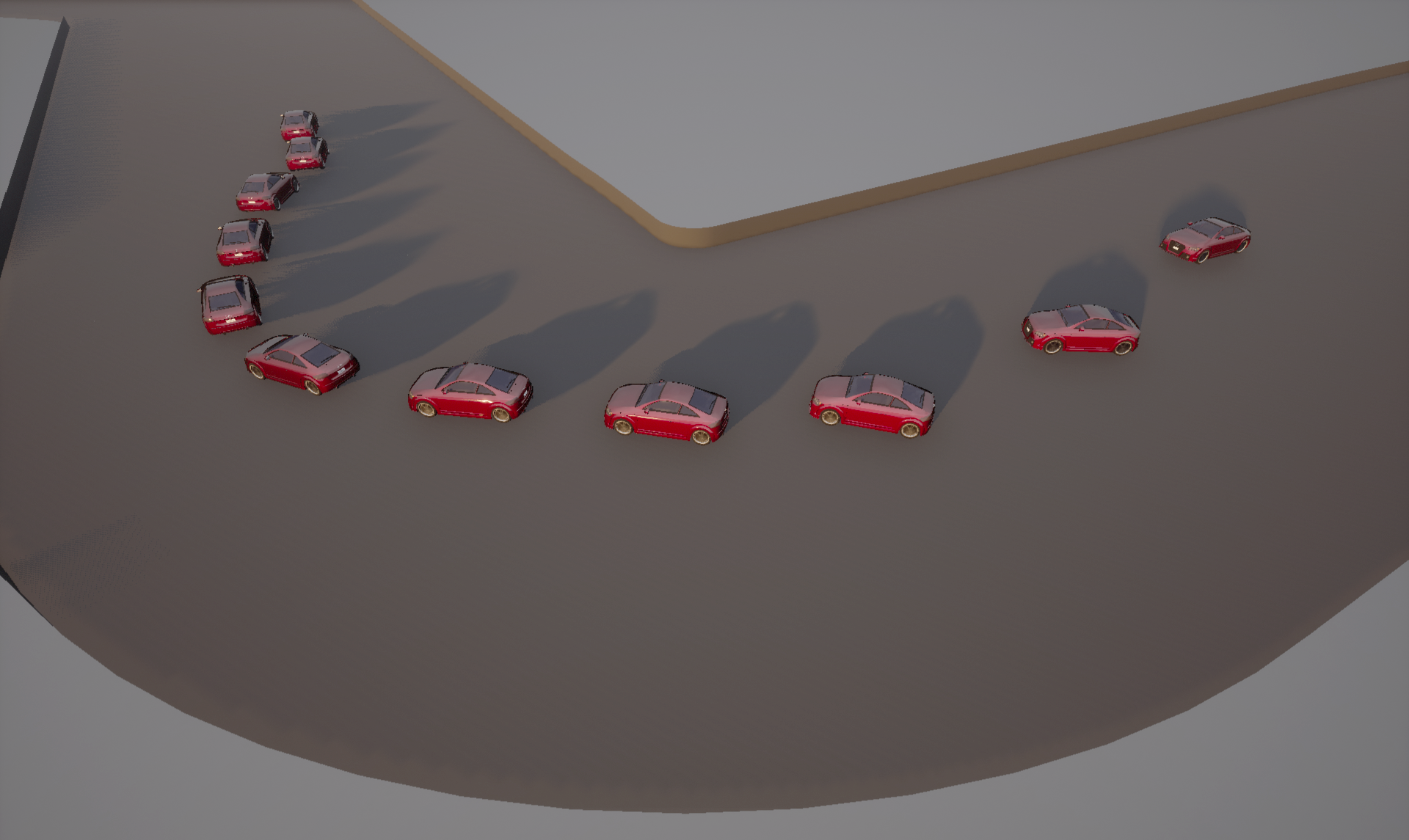}
    \caption{Drifting in racing circuit by a learned agent}
    \label{fig:mapC_task}
\end{figure}

%%%%%%%%%%%%%%%%%%%%%%%%%%%%%%%
% Fig: learning curve for mapC
\begin{figure}[!t]
    \centering
    \includegraphics[width=0.9\linewidth]{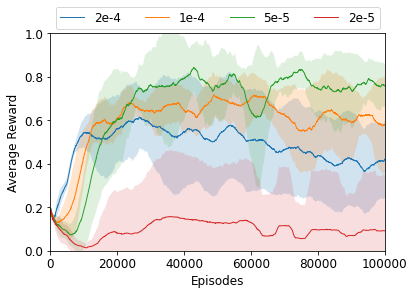}
    \caption{Training progress for safe drifting with different learning rate}
    \label{fig:mapC_training_progress}
\end{figure}

%
%\figref{fig:mapC_prob} is a visualization of the learned safety probability with various $\beta$ and $r$, and with $v_x = \SI{30}{m/s}$, $e=0$, and $\psi=0$. 
\Figref{fig:mapC_closed_loop}a shows the vehicle trajectories simulated using the learned agent. 
The cross mark ($\times$) in the figure shows the initial position of the vehicle, and 20 trajectories with different initial conditions of $\beta$ and $r$ are illustrated. 
%The position of the vehicle is shown by the cross mark ($\times$) in \figref{fig:mapC_closed_loop}a. 
%In the figure, the vehicle trajectories are shown for various initial conditions in $\beta \in [ -\SI{30}{deg}, -\SI{20}{deg} ]$ and  $r \in [ \SI{60}{deg/s}, \SI{70}{deg/s} ]$, where the safety probability is high. 
It can be confirmed that the learned policy is navigating through the corner without hitting the boundary of the track for different initial conditions. As shown in \figref{fig:mapC_closed_loop}b,  the slip angle and yaw rate take large values, and the vehicle is safely drifting while cornering. 
These safe behaviors have been learned only from sparse zero and one rewards. 
It is interesting that such a drifting maneuver can be learned by maximizing the safety probability without providing a specific reference trajectories or laborious reward shaping. 

% \begin{itemize}
%    \color{blue}
%     \item the vehicle is safer after the training (this is necessary but not sufficient)
%     \item agent can learn safe policy at higher slip angle and yaw rate than the state-of -the-art drifting paper do
%     \item In the above figure we are showing the learned agent navigating through a corner of the track with physics constraints.
% \end{itemize}

%%%%%%%%%%%%%%%%%%%%%%%%%%%%%%%
% Fig: safety probability
% \begin{figure}[!t]
%     \centering
%     \includegraphics[width=0.9\linewidth]{figs/mapC_safe_prob.png}
%     \caption{Safety probability of drifting states}
%     \label{fig:mapC_prob}
% \end{figure}

\begin{figure}[!t]
    \centering
     \subcaptionbox{Vehicle trajectories}{\includegraphics[width=0.9\linewidth]{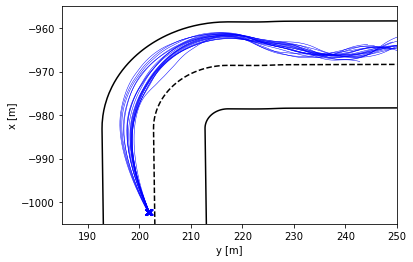}}\\[3mm]
    \subcaptionbox{Vehicle state}{ \includegraphics[width=0.9\linewidth]{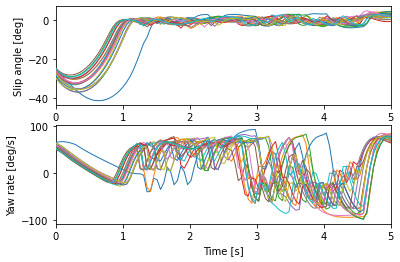}
    }
    \caption{Closed-loop simulation with learned policy}
    \label{fig:mapC_closed_loop}
\end{figure}

%%%%%%%%%%%%%%%%%%%%%%%%%%%%%%%%%%%%%%
\section{Conclusions}

This paper explored autonomous driving based on the learning of maximal safety probability by using Physics-informed Reinforcement Learning (PIRL). 
Safe drifting was achieved only from sparse binary rewards without providing a specific reference trajectory or laborious reward shaping.
Since this concept is related to forward invariance in the state space, the learned safety probability may be used to safeguard a nominal controller which does not necessarily consider the safety specifications. 
Future work of this paper includes verifying the accuracy of the learned safety probability, and testing the proposed framework in more diverse scenarios.

%%%%%%%%%%%%%%%%%%%%%%%%%%%%%%%%%
\section*{ACKNOWLEDGMENT}

The authors would like to thank Dvij Kalaria for providing vehicle model parameters for CARLA simulation.

\bibliographystyle{IEEEtran}
\bibliography{reference}

% \begin{thebibliography}{99}

% \bibitem{Abate2006} A. Abate, S. Amin, M. Prandini, J. Lygeros, and S. Sastry, “Probabilistic reachability and safe sets computation for discrete time stochastic hybrid systems,” in Proceedings of the 45th IEEE Conference on Decision and Control, Dec. 2006, pp. 258–263.

% \bibitem{Abate2008} A. Abate, M. Prandini, J. Lygeros, and S. Sastry, “Probabilistic reachability and safety for controlled discrete time stochastic hybrid systems,” Automatica , vol. 44, no. 11, pp. 2724–2734, Nov. 2008.

% \end{thebibliography}

\end{document}